\documentclass{article}
\usepackage[preprint]{log_2026}  

\usepackage{amssymb}
\usepackage{amsthm}
\usepackage{mathtools}
\usepackage{bm}
\usepackage{graphicx}
\usepackage{booktabs}
\usepackage{array}
\usepackage[round]{natbib}      

\theoremstyle{plain}
\newtheorem{theorem}{Theorem}
\newtheorem{proposition}{Proposition}
\newtheorem{lemma}{Lemma}

\theoremstyle{definition}
\newtheorem{definition}{Definition}

\newcommand{\R}{\mathbb{R}}

\newcommand{\onevec}{\mathbf{1}}
\newcommand{\phitrue}{\phi_{\mathrm{true}}}
\newcommand{\phihat}{\hat{\phi}}
\newcommand{\LG}{L_{G}}
\newcommand{\Lap}{L}
\newcommand{\interior}{I}
\newcommand{\bdry}{B}
\newcommand{\divg}{\operatorname{div}}
\newcommand{\spear}{\rho_{\mathrm{s}}}
\newcommand{\Tr}{\operatorname{tr}}
\DeclareMathOperator*{\argmin}{arg\,min}
\newcommand{\norm}[1]{\left\lVert #1 \right\rVert}

\AtEndPreamble{%
  \RequirePackage[nameinlink]{cleveref}%
  \crefname{theorem}{Theorem}{Theorems}%
  \crefname{proposition}{Proposition}{Propositions}%
  \crefname{lemma}{Lemma}{Lemmas}%
  \crefname{definition}{Definition}{Definitions}%
  \crefname{section}{Section}{Sections}%
  \crefname{figure}{Figure}{Figures}%
  \crefname{table}{Table}{Tables}%
  \crefname{equation}{Eq.}{Eqs.}%
}

\title[Gauge-Invariant, Parameter-Insensitive Regularization on Directed Graphs]%
{Gauge-Invariant, Parameter-Insensitive Regularization for Potential Recovery
from Flow on Directed Graphs}

\author[Mohammad Forouhesh]{%
Mohammad Forouhesh\\
Amirkabir University of Technology, Iran\\
\email{mforouhesh@aut.ac.ir}%
}

\begin{document}
\maketitle

\begin{abstract}
Recovering a latent potential from observed flow on a directed graph (a
discrete Poisson problem with Dirichlet boundaries) is ill-posed, and the
standard fix backfires: ridge regularization shrinks toward a gauge-meaningless
origin, collapsing and reversing the recovered ordering ($+0.81\to-0.42$
rank correlation against a planted ground truth). The gauge-invariant graph
Dirichlet energy removes the hazard and delivers \emph{parameter-insensitivity}:
the estimate is stable across four orders of magnitude in $\lambda$, whereas
ridge inverts the ordering for every $\lambda>0$. We prove the reduced solve is SPD
and preserves dynamic range exactly where ridge collapses it, and localize
absorbing boundaries from flow alone via a Poisson residual. The $H^1$ seminorm
is classical; what is new is the gauge diagnosis, the parameter-insensitivity it
buys, and an ablation showing the result is robust to the extraction method. On
three public clickstream corpora the gauge-invariant estimate retains
$28$--$41\%$ of the interior dynamic range while ridge collapses to as little as
$0.2\%$. The same gauge invariance carries into graph neural networks
(neutralizing the constant mode per layer prevents the oversmoothing that
collapses a deep directed GCN), linking this classical inverse problem to a
central question in graph learning.
\end{abstract}

\section{Introduction}
\label{sec:intro}

Regularizing an ill-posed inverse problem on a graph requires two decisions:
which penalty, and how strong. They are usually treated separately, assuming a
reasonable penalty degrades gracefully as $\lambda$ varies. Here that fails: the
conventional penalty does not degrade but \emph{inverts} the solution, so
$\lambda$ becomes a choice between a usable answer and a confidently wrong one. The setting is recovery of a
latent scalar \emph{potential} from observed directed flow: a traffic network
through flow counts, a supply chain through inventory movement, a navigation log
through transition frequencies: the discrete analogue of solving a Poisson
equation backwards from divergence to potential. We identify when this happens,
explain why, and show a different penalty removes it, leaving the answer
insensitive to $\lambda$ across orders of magnitude.

\paragraph{Setup and degeneracy.}
Let $G=(V,E,W)$ be a weighted directed acyclic graph with absorbing boundary
$\bdry\subset V$. A potential $\phi$ induces the flux
\begin{equation}
  q_{ij}=W_{ij}(\phi_j-\phi_i),
  \label{eq:constitutive-intro}
\end{equation}
the discrete analogue of Ohm's and Fick's laws, and the empirical data is the net
divergence $b_i=\sum_{j}F_{ji}-\sum_j F_{ij}$. Transporting this flux along the
edge directions gives the directed discrete Poisson equation $\Lap\phi=b$, where
$\Lap=D_{\mathrm{out}}-W$ is the \emph{one-sided} directed Laplacian
(\cref{sec:poisson-ops}); its interior kernel under Dirichlet data is the
absorption probability of the associated walk. Recovering $\phi$ from a noisy
empirical $b$ is
ill-posed ($\Lap\onevec=0$, low-divergence chains, finite data), so
regularization is unavoidable; the standard choice is general-form Tikhonov,
$\argmin_\phi\tfrac12\norm{\Lap\phi-b}_2^2+\lambda\norm{R\phi}_2^2$, with the
default ridge $R=I$.

\paragraph{The gauge mismatch.}
The potential is fixed only up to an additive constant, set when boundary values
are assigned. Ridge shrinks toward $0$, but $0$ is not distinguished in a
Dirichlet problem; it is wherever the boundary placed the origin. On a
positive-boundary graph the pull is asymmetric, dragging interior states toward
the abandon boundary regardless of their true position, collapsing the range
and, past a small $\lambda$, reversing the order. On a controlled instrument with
planted ground truth (\cref{sec:exp}), the ridge rank correlation falls from
$+0.81$ at $\lambda=0$ to ${\approx}-0.42$, inverting the ordering, with linear
correlation crossing zero (\cref{fig:sweep}); the only safe ridge setting is
$\lambda\to0$ (no regularization).

\paragraph{Parameter-insensitivity.}
The fix is to make the penalty blind to the gauge by taking $R$ to be the
incidence operator, so the penalty is the graph Dirichlet energy
$\phi^\top\LG\phi=\sum_{(u,v)\in E}w_{uv}(\phi_v-\phi_u)^2$ (the graph $H^1$
seminorm); we call the resulting estimator \emph{graph-Sobolev regularization}.
Because $\LG\onevec=0$, this penalizes differences, not amplitude, and is flat
along the gauge mode. The consequence is practical: the estimate is
origin-invariant and stable in $\lambda$, holding rank correlation at $+0.81$ and
linear correlation in $[+0.76,+0.85]$ across $\lambda\in[10^{-3},10]$, so the
strength cannot be misset. Where ridge forces a knife-edge choice, the Dirichlet
energy removes it.

\paragraph{What is, and is not, new.}
The penalty is old: the $H^1$ seminorm is the classical alternative to the
identity seminorm (general- vs.\ standard-form Tikhonov
\citep{hansen1998,tikhonov1977}), and $\phi^\top L\phi$ is standard
in graph signal processing and semi-supervised learning
\citep{shuman2013,zhou2004,belkin2003,zhu2003}. What is new is: (i) those priors
interpolate a \emph{partially observed} node signal, whereas ours is an
\emph{inhomogeneous} inverse problem deconvolving a potential from its divergence
with Dirichlet boundaries, where the decisive property is gauge invariance, not
smoothness; (ii) ridge is not merely suboptimal but
\emph{actively harmful}, inverting the ordering, with the gauge diagnosis
explaining why; (iii) the parameter-insensitivity this yields; (iv) a
Poisson-residual boundary diagnostic; and (v) a behavioral-flow application with
a planted-ground-truth instrument. Edge-preserving gauge-invariant penalties
(total variation, the $p$-Laplacian \citep{rudin1992}) are discussed in
\cref{sec:related}.

\paragraph{Contributions.}
\textbf{(1)} Parameter-insensitivity: the gauge-invariant
penalty makes the recovered potential stable across four orders of magnitude in
$\lambda$ (\cref{sec:sobolev}); \textbf{(2)} the ridge inversion and its gauge
diagnosis ($+0.81\!\to\!-0.42$ rank correlation); \textbf{(3)} discrete
guarantees: gauge invariance (\cref{prop:gauge}), the rank structure of the
reduced system and its minimum-norm solve (\cref{thm:spd}), and exact range
preservation on chains (\cref{thm:range}); and \textbf{(4)} a reproducible
pipeline\footnote{Code: \url{https://github.com/MohammadForouhesh/gauge-flow-recovery}}
with a Poisson-residual boundary diagnostic (\cref{sec:sink}), an
extraction ablation, real-corpus validation, and a downstream task where the
gauge-invariant potential is a usable node feature and ridge's is not
(\cref{sec:exp,sec:real}). We are precise about scope: the penalty guarantees
preservation, not signal: whatever ordering the data supports is preserved
across $\lambda$, not destroyed.

\section{Related Work}
\label{sec:related}

\paragraph{Inverse problems and flows on graphs.}
Recovering a potential from observed flow is the discrete counterpart of an
elliptic inverse problem. The electrical-network analogy is classical: a
resistor network with prescribed boundary potentials obeys a discrete Laplace
equation whose harmonic functions encode absorption probabilities of the
associated walk \citep{doyle1984,grady2006}. Closest to our setting,
\citet{jia2019edgeflows} recover edge flows by Hodge-based semi-supervised
learning, and \citet{schaub2018flow} denoise edge flows by projecting onto Hodge
gradient/curl subspaces. We differ in recovering the \emph{node potential} through the
directed forward operator with Dirichlet boundaries, and in diagnosing the
gauge-induced \emph{inversion} that the identity seminorm causes, a failure
mode about the regularizer, not the representation.

\paragraph{Regularization and graph signal processing.}
Tikhonov regularization is the default stabilizer for ill-posed systems
\citep{tikhonov1977,hansen1998}, the seminorm matrix encoding prior structure;
total-variation and graph trend-filtering penalties trade quadratic smoothness
for sharp transitions \citep{rudin1992,wang2016gtf}. Whether the penalty
seminorm annihilates the forward operator's null space governs the estimator's
bias. Known in principle, we make it concrete and consequential for the directed
Dirichlet problem: the identity seminorm does not annihilate the gauge mode, and
the resulting bias is not lost resolution but a reversal of order. The same
quadratic form $\phi^\top\LG\phi$ is the standard smoothness functional in graph
signal processing and Laplacian semi-supervised learning
\citep{shuman2013,belkin2003,zhu2003,zhou2004}, the ``trivial'' structural choice
that sheaf and connection Laplacians generalize \citep{bodnar2022sheaf}, but
there for a signal \emph{partially observed on the nodes}, not, as here, an
inhomogeneous inverse problem where gauge invariance is decisive.

\paragraph{Directed graphs and the Hodge decomposition.}
Directed graph neural networks encode direction in a complex Hermitian
(magnetic) Laplacian \citep{zhang2021magnet,he2022msgnn} or directionality-aware
message passing \citep{rossi2023edge}, and directed graph signal processing
studies edge directionality \citep{marques2020directed}; these target node/edge
representation learning,
whereas we keep a \emph{real} directed Laplacian as a forward operator and
regularize its inverse. The Helmholtz--Hodge decomposition
\citep{jiang2011,lim2020,schaub2018} splits an edge flow into gradient, curl,
and harmonic parts; we use the gradient only to orient flow into an acyclic
support, which \cref{sec:exp-structure,sec:exp-ablation} show is not
load-bearing. Graph total variation, trend filtering, and the $p$-Laplacian
\citep{rudin1992,wang2016gtf} are gauge-invariant alternatives; the dividing line
for the ridge pathology is gauge invariance, which they share and the identity
seminorm lacks, not the exponent.

\paragraph{Connection to graph representation learning.}
Operationally the gauge-invariant potential is a directed-graph \emph{node
feature}, useful downstream where the ridge potential is not
(\cref{sec:real-downstream}). Its collapse is an instance of \emph{oversmoothing}: repeated propagation drives node representations toward the constant mode
\citep{bodnar2022sheaf}, just as the magnitude penalty collapses the potential's
range. We verify it, a vanilla directed GCN's conversion AUC falling from $0.86$
($2$ layers) to $0.81$ ($32$) as node energy collapses; and the same gauge
invariance fixes it: neutralizing the constant mode per layer (the analogue of
$\LG\onevec=0$) holds AUC flat at $0.85$, recovering PairNorm and
Dirichlet-energy-constrained networks \citep{zhao2020pairnorm,zhou2021dirichlet}
from one principle. Where magnetic- and sheaf-Laplacian networks
\citep{zhang2021magnet,bodnar2022sheaf} learn embeddings, we recover one
interpretable scalar per node through an explicit forward operator.

\section{From Raw Flows to an Acyclic Support}
\label{sec:dag}

The inverse problem is posed on a DAG, but raw flow is neither acyclic nor
purely gradient: it carries reciprocal traffic, transient loops and a solenoidal
component no potential explains. Extraction is upstream machinery, and
\cref{sec:exp-ablation} shows the claims are \emph{invariant} to how it is done,
so we keep this brief and defer mechanics to \cref{app:extract}.

\paragraph{Flow and dominance.}
\label{sec:dag-flow}\label{sec:dag-dom}
Sessions are state sequences; collapsing consecutive repeats and counting
transitions gives the empirical flow $F_{uv}$, used as the conductance
$W_{uv}=F_{uv}$. Reciprocal traffic is filtered by a \emph{dominance} test: for
$\rho\ge1$, the pair $(u,v)$ is dominant if $F_{uv}\ge\rho F_{vu}$.

\paragraph{Orientation.}
\label{sec:dag-orient}
An acyclic support needs an orientation of the dominant edges. When session
order is available, the default is a cheap \emph{topological sort}: order states
by mean visit position (or net inflow), keeping dominant edges that respect the
order. This needs no linear solve, and \cref{sec:exp-ablation} shows it recovers
the potential at least as well as the alternatives.

When only an aggregate flow matrix is available (no session order to sort by),
we instead orient by the discrete Helmholtz--Hodge projection: the gradient
potential $\phi_0$ solving $G^\top G\,\phi_0=G^\top\omega$ for the skew flow
$\omega_{uv}=F_{uv}-F_{vu}$, with edges oriented by increasing $\phi_0$
\citep{jia2019edgeflows,schaub2018flow}. It needs only the flow, but the ablation
finds it no more accurate than the topological sort and slower; we keep it solely
for the order-free setting. Retention rules for both are in \cref{app:extract}.

\paragraph{Acyclicity and the Poisson right-hand side.}
\label{sec:dag-nilpotent}\label{sec:dag-div}
\begin{proposition}[Acyclicity]
\label{prop:acyclic}
Either orientation induces a topological order of the retained support
$G_\delta=(V,E_\delta)$; hence $G_\delta$ is acyclic.
\end{proposition}
Let $A$ be the substochastic transition operator on $G_\delta$ (outgoing weights
normalized, absorbing rows at sinks). Since $G_\delta$ is a DAG, $A$ is
nilpotent.
\begin{proposition}[Nilpotency]
\label{prop:nilpotent}
If the longest directed path in $G_\delta$ has length $L$, then $A^m=0$ for some
$m\le L+1$.
\end{proposition}
Nilpotency is the algebraic certificate that extraction succeeded: a single
residual cycle would make $A^m\ne0$ for all $m$; we verify $\norm{A^m}_F=0$ to
machine precision on the instrument (\cref{sec:exp}). The empirical divergence
$b_i=\sum_j F_{ji}-\sum_j F_{ij}$ is the Poisson right-hand side
(\cref{sec:poisson}).
\Cref{prop:acyclic,prop:nilpotent} are proved in \cref{app:proofs}.

\section{The Poisson Inverse Problem}
\label{sec:poisson}

\paragraph{Two Laplacians.}
\label{sec:poisson-ops}
Two distinct operators appear, and keeping them apart is essential. The
\emph{forward} operator $\Lap$ is the \emph{one-sided} (out-edge) directed
Laplacian of $G_\delta$,
\begin{equation}
  (\Lap\phi)_u=\sum_{v\,:\,u\to v} w_{uv}\,(\phi_u-\phi_v),
  \qquad \Lap = D_{\mathrm{out}}-W,
  \label{eq:directed-lap}
\end{equation}
not symmetric, and summing only over edges \emph{leaving} $u$. With
$\Lap=D_{\mathrm{out}}(\mathrm{Id}-A)$ for $A$ the row-normalized transition
operator (\cref{sec:dag-nilpotent}), the interior equation $\Lap\phi=0$ reads
$\phi=A\phi$: under the Dirichlet data of \cref{sec:poisson-dirichlet} its
solution is the absorption probability
$\phi(u)=\Pr[\text{conversion before abandon}\mid u]$, exactly the field the
instrument of \cref{sec:exp} plants. The \emph{regularization} operator $\LG$
(\cref{def:sobolev}) is the \emph{symmetric} conductance-weighted Laplacian of
the \emph{undirected} support; it never acts as a forward operator. Both use the
same conductances $w_{uv}=F_{uv}$: the asymmetry of $\Lap$ carries flow
direction, the symmetry of $\LG$ makes the penalty a squared-difference energy.
Throughout $\Lap_{\interior\interior}^\top\Lap_{\interior\interior}$ is the Gram
matrix of the interior block, not the interior block of $\Lap^\top\Lap$.

\paragraph{Constitutive law.}
\label{sec:poisson-var}
The flux law \cref{eq:constitutive-intro} is the optimality condition of a
minimum-dissipation principle, pinning down the flux the operator transports.
\begin{proposition}[Minimum dissipation]
\label{prop:dissipation}
Among all edge flows $q$ with prescribed divergence $\divg(q)=b$, the one
minimizing $\tfrac12\sum_{(i,j)\in E}W_{ij}^{-1}q_{ij}^2$ is a gradient flow:
there exists a potential $\phi$ with $q_{ij}=W_{ij}(\phi_j-\phi_i)$.
\end{proposition}
\Cref{prop:dissipation} justifies the constitutive law, i.e.\ that a potential
exists at all; which operator transports it is then fixed by the edge
directions, giving $\Lap$ of \cref{eq:directed-lap}.

\paragraph{Least-squares objective and its degeneracies.}
\label{sec:poisson-ls}
Given a noisy empirical $b$, the natural estimator minimizes the flow residual
\begin{equation}
  E(\phi) = \tfrac12\norm{\Lap\phi-b}_2^2,
  \qquad \Lap^\top\Lap\,\phi=\Lap^\top b.
  \label{eq:ls}
\end{equation}
Two degeneracies make \cref{eq:ls} ill-posed even on exact $b$: $\Lap\onevec=0$
leaves $\phi$ determined only up to an additive gauge, and interior chains
carry no divergence, which flattens the solution.
\begin{lemma}[Chain saturation]
\label{lem:saturation}
On an interior chain $v_1\to\cdots\to v_k\to s$ with $b_{v_i}=0$ for every
interior $v_i$ and $\phi(s)$ pinned, any solution of $\Lap\phi=b$ satisfies
$\phi(v_1)=\cdots=\phi(v_k)=\phi(s)$.
\end{lemma}
\Cref{lem:saturation} is why the regularizer choice matters: a chain encodes the
\emph{ordering} of its states, but the unregularized solve flattens it. A penalty
that further compresses the chain (as ridge does) erases the ordering; one biasing
toward a smooth ramp recovers it, the distinction \cref{sec:sobolev} makes precise.

\paragraph{Dirichlet boundaries fix the gauge.}
\label{sec:poisson-dirichlet}
We resolve the constant-mode degeneracy by pinning the potential on the
absorbing boundary $\bdry$: \emph{conversion} sinks at $1$ and an
\emph{abandon} sink at $0$. Partitioning $V=\interior\cup\bdry$ and writing
$\Lap,b$ in block form, the reduced unknown $\phi_\interior$ has effective
right-hand side $b_\interior-\Lap_{\interior\bdry}\phi_\bdry$ and operator
$\Lap_{\interior\interior}^\top\Lap_{\interior\interior}$. This removes the
gauge freedom but, as \cref{sec:sobolev} shows, does not prevent a magnitude
penalty from biasing the interior toward the boundary-induced origin.
\Cref{prop:dissipation,lem:saturation} are proved in \cref{app:proofs}.

\section{Graph-Sobolev Regularization}
\label{sec:sobolev}

Ridge (standard-form Tikhonov) regularization inverts the recovered ordering
(\cref{sec:sob-pathology}); the gauge-invariant Dirichlet-energy penalty
(\cref{sec:sob-def}) has three properties: gauge invariance
(\cref{sec:sob-gauge}), a fully characterized reduced system
(\cref{sec:sob-spd}), and exact range preservation on chains
(\cref{sec:sob-range}). \Cref{sec:sob-empirics} reports the headline
\emph{parameter-insensitivity}: the estimate is stable across four orders of
magnitude in $\lambda_1$, ridge only at $\lambda_1\!\to\!0$. Proofs are in
\cref{app:proofs}.

\subsection{The ridge inversion and the gauge mismatch}
\label{sec:sob-pathology}
The standard estimator augments \cref{eq:ls} with a magnitude penalty,
\begin{equation}
  \phihat_{\mathrm{Tik}} = \argmin_\phi\ \tfrac12\norm{\Lap\phi-b}_2^2
  + \lambda_1\norm{\phi}_2^2,
  \label{eq:tikhonov}
\end{equation}
subject to the Dirichlet boundary. The penalty pulls every interior value toward
$0$, but on a Dirichlet problem $0$ is not distinguished: it is merely where the
boundary placed the abandon sink. The estimator therefore carries a bias toward
that boundary set by the labelling, not the data, and
\cref{prop:tik-collapse} makes the effect exact: the interior range shrinks like
$1/\lambda_1$, so the ordering flattens and then reverses.

\begin{proposition}[Magnitude shrinkage collapses the range]
\label{prop:tik-collapse}
Let $\phihat_{\mathrm{Tik}}(\lambda_1)$ solve \cref{eq:tikhonov} on an
interior chain of length $L$ with the sink pinned at $1$ and unit divergence
injected at the free source. Then the interior dynamic range
$\Delta\phihat_{\mathrm{Tik}}=\max_\interior\phihat-\min_\interior\phihat$
satisfies $\Delta\phihat_{\mathrm{Tik}}(\lambda_1)\le C/\lambda_1\to 0$ as
$\lambda_1\to\infty$, with a constant $C$ \emph{independent of the chain
length $L$}: the magnitude penalty extracts no benefit from a longer chain.
\end{proposition}

\Cref{prop:tik-collapse} does not claim the rank correlation tends to $-1$;
empirically it settles to a negative plateau near $-0.42$ (\cref{fig:sweep}),
because the collapsed order is dominated by residual divergence structure
anticorrelated with the planted field. The damage is decisive: ridge
\emph{inverts} the ordering for every $\lambda_1>0$, linear correlation crossing
to zero (\cref{tab:sweep}).

\subsection{The graph-Sobolev penalty}
\label{sec:sob-def}
\begin{definition}[Graph-Sobolev objective]
\label{def:sobolev}
Let $\LG$ be the symmetric, conductance-weighted graph Laplacian of the
undirected support of $G_\delta$, so that
\begin{equation}
  \phi^\top\LG\phi=\sum_{(u,v)\in E} w_{uv}(\phi_v-\phi_u)^2 .
  \label{eq:dirichlet-energy}
\end{equation}
The graph-Sobolev estimator is
\begin{equation}
  \phihat_{\mathrm{Sob}} = \argmin_\phi\ \tfrac12\norm{\Lap\phi-b}_2^2
  + \lambda_1\,\phi^\top\LG\phi,
  \label{eq:sobolev-obj}
\end{equation}
subject to the Dirichlet boundary.
\end{definition}
We use the \emph{weighted} energy: each edge contributes
$w_{uv}(\phi_v-\phi_u)^2$ with the flux conductance of \cref{eq:constitutive-intro} (the penalty in the
released code). The unweighted form $\LG=G^\top G$ is the case $w_{uv}\equiv 1$.

\subsection{Gauge invariance}
\label{sec:sob-gauge}
A regularizer $R$ is \emph{gauge-invariant} if $R(\phi+c\onevec)=R(\phi)$ for
all $c\in\R$.
\begin{proposition}[Gauge invariance]
\label{prop:gauge}
The Dirichlet energy $\phi^\top\LG\phi$ is gauge-invariant, whereas the
magnitude penalty $\norm{\phi}_2^2$ is not.
\end{proposition}
The conceptual heart: since $\LG\onevec=0$ the Dirichlet
energy penalizes only \emph{differences} and is indifferent to the origin,
whereas $\norm{\phi}_2^2$ imposes a preferred mean-zero gauge that has no
meaning in the Dirichlet problem and fights the boundary conditions.

\subsection{Well-posedness}
\label{sec:sob-spd}
\begin{theorem}[Reduced system: definiteness and rank structure]
\label{thm:spd}
Let $\lambda_1>0$, let
\begin{equation}
  M = \Lap_{\interior\interior}^\top\Lap_{\interior\interior}
  + \lambda_1\,(\LG)_{\interior\interior},
  \label{eq:M}
\end{equation}
and let $\mathcal{C}$ be the connected components of the \emph{undirected}
support of $G_\delta$. Then $M$ is symmetric PSD and block diagonal over
$\mathcal{C}$, with
\begin{equation}
  \ker M=\operatorname{span}\{\onevec_C:\ C\in\mathcal{C},\ C\cap\bdry=\emptyset\}.
  \label{eq:kerM}
\end{equation}
Hence $M$ is positive definite on every component meeting $\bdry$; in
particular, if every interior node has a directed path to $\bdry$ then $M$ is
SPD and \cref{eq:sobolev-obj} has a unique interior solution obtainable by
Cholesky. In general, if some component meets $\bdry$ and contains an interior
state then $\onevec\notin\ker M$, and the minimum-norm solution $\phihat=M^{+}r$
\textbf{(i)} is that unique solution on components meeting $\bdry$,
\textbf{(ii)} takes the zero-mean representative on each boundary-free
component, and \textbf{(iii)} is exactly gauge-invariant on the former: shifting
all Dirichlet values by $c$ shifts $\phihat$ there by exactly $c\onevec$.
\end{theorem}
Real extractions violate the path-to-boundary hypothesis (some retained states
reach no sink, and the pipeline carries states extraction retains no edge for),
so \cref{eq:kerM} is the operative case. The deficiency is bookkeeping, not
pathology: boundary-free components are almost entirely \emph{isolated} states
($127$ of $128$ on the instrument of \cref{sec:exp}). Since $\onevec$ lies in the
\emph{range} of $M$, minimum-norm preserves the gauge exactly, whereas a diagonal
floor would pin the near-null directions to roundoff and destroy origin
invariance (\cref{tab:gauge}). Proved in \cref{app:proofs}.

\subsection{Range preservation on chains}
\label{sec:sob-range}
\begin{theorem}[Range preservation, exact]
\label{thm:range}
On the length-$L$ constant-conductance chain of \cref{prop:tik-collapse}, with
the trace-normalized Dirichlet-energy penalty (\cref{sec:sob-trace}), the
graph-Sobolev estimator has interior dynamic range exactly
\begin{equation}
  \Delta\phihat_{\mathrm{Sob}}(\lambda_1)
  = \frac{2L+1}{\,2L+1+2\lambda_1\,}
  = \frac{1}{1+\lambda_1\,\kappa(L)},
  \qquad \kappa(L)=\frac{2}{2L+1}=O(1/L).
  \label{eq:range-bound}
\end{equation}
Hence for any fixed $\lambda_1$ the retained range tends to the full harmonic
range $\Delta\phitrue=1$ as $L\to\infty$. Contrasted with
\cref{prop:tik-collapse}, the Sobolev range depends on $\lambda_1$ only through
$\lambda_1/(2L+1)$, whereas ridge collapses like $1/\lambda_1$ independently of
$L$: a longer chain helps the Dirichlet-energy estimate and never helps ridge.
\end{theorem}
On a length-$30$ chain \cref{eq:range-bound} gives $0.968$ at $\lambda_1=1$ and
$0.753$ at $\lambda_1=10$, versus $0.536$ and $0.091$ for ridge
(\cref{fig:chain}).

\subsection{Parameter-insensitivity: preservation versus collapse}
\label{sec:sob-empirics}
We compare the two penalties on the synthetic instrument of \cref{sec:exp}
(planted harmonic $\phitrue$, $\lvert V\rvert=277$, five sinks), scoring on the
\emph{determined} interior: the $n=95$ well-visited states ($\ge100$ visits)
with a directed path to the boundary (nodes without one are undetermined by the
data, \cref{thm:spd}). Rank correlation is computed on estimates rounded to
numerical tolerance so the saturated near-zero cluster (\cref{lem:saturation})
ties deterministically rather than by roundoff.
\Cref{fig:sweep,tab:sweep} report the headline: robustness to $\lambda_1$, not a
single best setting. At $\lambda_1=0$ both estimators recover the planted
ordering ($\spear=+0.81$, $r=+0.76$); as $\lambda_1$ grows, graph-Sobolev holds
this unchanged ($\spear$ constant at $+0.807$ across
$\lambda_1\in[10^{-3},10]$, $r\in[+0.76,+0.85]$), while ridge \emph{inverts} the
instant regularization is applied, dropping to a plateau near $-0.42$ with
linear correlation crossing zero. The ${\approx}+1.23$ gap is itself \emph{flat
in $\lambda_1$}: no ridge setting recovers the lost structure
(\cref{fig:scatter}, \cref{app:figures}). The same holds for \emph{top-$k$
power}: NDCG@$5$ stays in $[0.97,1.0]$ under Sobolev but drops to $0.76$ for
ridge, so parameter-insensitivity shows in top-of-ranking quality, not only
correlation.

\begin{table}[t]
\centering
\caption{Recovery and top-$5$ ranking power versus $\lambda_1$ (determined
interior, $n=95$; HHD extraction). The gauge-invariant penalty is insensitive to
$\lambda_1$ (rank correlation constant at $+0.81$ and NDCG@$5$ in $[0.97,1.0]$
across four orders of magnitude), while ridge \emph{inverts} the ordering for
every $\lambda_1>0$ (Spearman to ${\approx}-0.42$, NDCG@$5$ from $0.97$ to
$0.76$). Metrics vs.\ the planted $\phitrue$ (graded relevance for NDCG); Spearman
uses estimates rounded to tolerance so the saturated near-zero cluster does not
break ties by roundoff.}
\label{tab:sweep}
\begin{tabular}{r ccc ccc}
\toprule
& \multicolumn{3}{c}{graph-Sobolev} & \multicolumn{3}{c}{Tikhonov}\\
\cmidrule(lr){2-4}\cmidrule(lr){5-7}
$\lambda_1$ & Spearman & Pearson & NDCG@5 & Spearman & Pearson & NDCG@5\\
\midrule
$0$        & $+0.810$ & $+0.758$ & $0.966$ & $+0.810$ & $+0.758$ & $0.966$\\
$10^{-3}$  & $+0.807$ & $+0.760$ & $0.966$ & $-0.450$ & $-0.055$ & $0.756$\\
$10^{-2}$  & $+0.807$ & $+0.767$ & $0.982$ & $-0.424$ & $-0.005$ & $0.756$\\
$10^{-1}$  & $+0.807$ & $+0.783$ & $0.982$ & $-0.418$ & $+0.001$ & $0.756$\\
$1$        & $+0.807$ & $+0.831$ & $0.999$ & $-0.422$ & $+0.001$ & $0.756$\\
$10$       & $+0.807$ & $+0.849$ & $0.982$ & $-0.551$ & $+0.001$ & $0.756$\\
\bottomrule
\end{tabular}
\end{table}

\begin{figure}[t]
\centering
\includegraphics[width=0.74\linewidth]{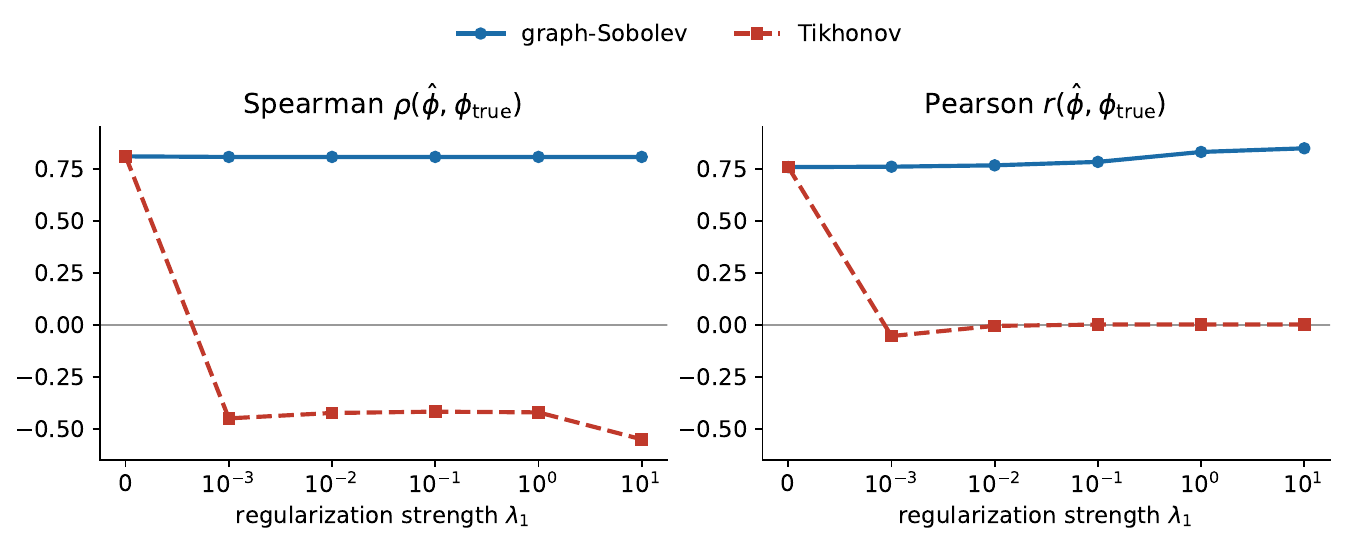}
\caption{Recovery of the planted potential versus $\lambda_1$. Graph-Sobolev
(blue) preserves both rank and linear correlation across four orders of
magnitude; Tikhonov (red) inverts the rank correlation to ${\approx}-0.42$ and
collapses the linear correlation to zero. Determined interior, $n=95$.}
\label{fig:sweep}
\end{figure}

\section{Boundary Identification from Flow Geometry}
\label{sec:sink}

When the absorbing boundary is unknown, ranking states by the recovered
potential fails on multi-sink graphs: with several sinks pinned at the same
value, the interior potential is rank-deficient and sinks do not separate from
high-potential interior states. The \emph{Poisson residual}
$r_u=\lvert[\Lap\phi]_u-b_u\rvert$ succeeds instead.
\begin{theorem}[Residual concentration at absorbing nodes]
\label{thm:residual}
Let $u$ be an absorbing sink with no outgoing DAG edges. Then $[\Lap\phi]_u=0$
exactly, so $r_u=\lvert b_u\rvert$ equals the empirical inflow at $u$. If every
interior node satisfies the Poisson equation to tolerance $\eta$ and
$\min_{u\in\bdry}\lvert b_u\rvert>\eta$, then every sink ranks strictly above
every interior node under $r$.
\end{theorem}
A sink is where mass accumulates without modeled outflow: the operator predicts
zero net flux but the data shows large inflow, and that mismatch is the residual;
interior nodes, where the solve balances flux, have small residual. On the
instrument ($\lvert V\rvert=277$, five planted sinks) the residual ranks all five
true sinks in the top seven (worst rank $7$), while potential ranking scatters
them to worst rank $108$; the separation is invariant across abandonment rates
$\{0.05,0.20,0.50\}$, a regime-invariant boundary detector. This matters because \emph{missing} a true
sink is the one boundary error that costs recovery ($\spear\!\to\!+0.15$--$0.39$);
spurious or mis-valued sinks are nearly harmless (\cref{tab:badbnd}). Proved in
\cref{app:proofs}.

\section{Synthetic Validation}
\label{sec:exp}

We validate on a controlled instrument: recovery can only be scored against a
known potential, which no observational corpus provides. The instrument plants a
ground-truth potential, samples flow from it, and hands the algorithm only the
flow; every number comes from the released generator at one fixed configuration.

\paragraph{The instrument.}
\label{sec:exp-instrument}
The generator builds a multi-branch conversion funnel: a source feeds shared
upper states; each of $K=5$ branches is a chain terminating in a conversion
sink; one abandon sink absorbs drop-off at a depth-decaying rate. The ground
truth is the harmonic extension under
$\{\text{conversion}=1,\text{abandon}=0\}$, i.e.\
$\phitrue(u)=\Pr[\text{conversion before abandon}\mid u]$, exactly the object the
Poisson solve targets. The default point uses $20$ shared states, $50$ interior
states per branch, abandonment $0.05$ and $50{,}000$ sessions, giving
$\lvert V\rvert=277$ and sink entropy $H(\pi)\approx2.27$ bits; sessions are
bit-for-bit reproducible at a given seed.

\paragraph{Structural certificates and recovery.}
\label{sec:exp-structure}\label{sec:exp-recovery}
The extracted operator is nilpotent ($\norm{A^m}_F=0$ at the predicted $m=7$),
confirming \cref{prop:nilpotent} (no residual cycle). The Hodge gradient
component $\phi_0$ alone does \emph{not} recover the planted direction
(conductance-weighted cosine $-0.06$): an honest negative showing the gradient
projection is a structural device, not an estimator. Recovery is the regularized
solve's job (\cref{sec:sob-empirics}): graph-Sobolev holds rank correlation
$+0.81$ while ridge \emph{inverts} to ${\approx}-0.42$, and preserves $97\%$ of
the chain range against ridge's $54\%$ (\cref{fig:chain}), independent of the
Hodge step (\cref{sec:exp-ablation}).

\subsection{Does the Hodge projection matter? An extraction ablation}
\label{sec:exp-ablation}
We ask whether any acyclic support, in particular a cheap topological sort,
serves as well as the Hodge projection, replacing it with two
dominance-thresholded topological sorts, the rest of the pipeline fixed. The
orders come from data alone: \emph{visit-position}
(mean fractional session position) and \emph{net-inflow}
($\sum_v(F_{vu}-F_{uv})$); an edge is kept when net-forward, dominant, and
order-consistent (acyclic by construction).

\Cref{tab:ablation} reports recovery under each extraction. First, the
contrast, the paper's central claim, is \emph{invariant to the extraction}: graph-Sobolev recovers a strong positive
ordering under all three supports while ridge \emph{inverts} under all three
($\spear=-0.42,-0.32,-0.56$; margin $+1.23$ to $+1.52$). Second, Hodge is
\emph{not} the best extraction: both topological sorts give higher Sobolev rank
correlation ($+0.96,+0.99$) than Hodge ($+0.81$). We treat the topological sort
as the default, keeping Hodge for its generality (only the flow matrix) and the
nilpotency certificate, not accuracy. Either way the contribution is isolated in
the regularizer, not a Hodge artifact.

\begin{table}[t]
\centering
\caption{Extraction ablation ($\lambda_1=1$, determined interior: $n=95$ for
Helmholtz--Hodge, $n=150$ for the topological sorts, which reach the boundary
from every well-visited state). The regularizer contrast holds under every
acyclic support: graph-Sobolev recovers a strong positive ordering while ridge
inverts. Hodge is not required for recovery; all supports are exactly acyclic.}
\label{tab:ablation}
\begin{tabular}{l c cc c}
\toprule
& & \multicolumn{2}{c}{Spearman vs.\ $\phitrue$} & \\
\cmidrule(lr){3-4}
Extraction & edges & graph-Sobolev & Tikhonov & margin\\
\midrule
Helmholtz--Hodge (\cref{sec:dag}) & $185$ & $+0.807$ & $-0.422$ & $+1.23$\\
Topological sort (visit-position) & $664$ & $+0.995$ & $-0.324$ & $+1.32$\\
Topological sort (net-inflow)     & $375$ & $+0.961$ & $-0.556$ & $+1.52$\\
\bottomrule
\end{tabular}
\end{table}

\subsection{Regime dependence and scope}
\label{sec:exp-regime}
The \emph{magnitude} of the recoverable correlation is a regime property, not the regularizer's: at low abandonment the unregularized solve recovers
$\spear\approx+0.81$, while as abandonment rises the well-visited interior
shrinks and recovery degrades (to $+0.23$ at abandonment $0.5$). The
regularizer's role is invariant: graph-Sobolev
\emph{preserves} whatever ordering the unregularized solve attains and ridge
\emph{degrades} it; what is regime-dependent is whether there is signal to
preserve. Under noise on the divergence the \emph{unregularized} solve collapses
at $5\%$ while graph-Sobolev holds $\spear=+0.62$ even at noise equal to the
signal RMS, falling back on penalty and boundary (\cref{tab:noise}).

This is not an artifact of one configuration: across twelve instrument
variants (five seeds plus single-axis perturbations of branch count, depth,
abandonment, and sink entropy), the margin at $\lambda_1=1$ is positive in
\emph{all twelve} under both extractions, never falling below $+0.45$
(\cref{tab:robustness}, \cref{app:figures}).

\paragraph{Scaling.}
\label{sec:exp-scaling}
Because the reduced system is symmetric PSD (\cref{thm:spd}), the solve admits
sparse conjugate gradients (CG), converging to the same minimum-norm estimate as
the dense path. On a planted funnel with $\lvert V\rvert=10{,}202$ graph-Sobolev
recovers $\spear=+0.94$ while ridge inverts to $-0.18$, CG returning in
$0.08$\,s where the dense solve is killed for memory. CG matches the dense
solve to $10^{-9}$, scales near-linearly in edges, and reaches $4\times10^{5}$
interior nodes in $8.5$\,s (\cref{tab:scaling}, \cref{app:figures}); the limit
is set by support sparsity, not node count.

\section{Real-Data Validation}
\label{sec:real}

No clickstream corpus carries a ground-truth potential, so we validate claims
that need no planted field: the recovered potential is interpretable
(\cref{sec:real-interp}); the regularizer contrast holds \emph{internally}
(\cref{sec:real-contrast}); and the residual concentrates at the known boundary
while $\phi$ is stable under resampling (\cref{sec:real-stability}).
We use three public corpora, each reduced to an event-type state space
(\cref{sec:dag-flow}) and summarized in \cref{tab:real-datasets}:
\textbf{RetailRocket} (a tiny funnel), \textbf{Trivago} (a medium action graph),
and \textbf{OTTO} (a large multi-sink graph).

\begin{table}[t]
\centering
\caption{Real corpora, reduced to event-type state spaces. Conversion is the
fraction of sessions terminating at a conversion sink.}
\label{tab:real-datasets}
\begin{tabular}{l r r r l}
\toprule
Dataset & Sessions & States & Conversion & Conversion sink\\
\midrule
RetailRocket & $41{,}184$  & $4$   & $33.0\%$ & \texttt{transaction}\\
Trivago      & $640{,}728$ & $11$  & $92.6\%$ & \texttt{clickout item}\\
OTTO         & $186{,}271$ & $105$ & $27.3\%$ & \texttt{orders}\\
\bottomrule
\end{tabular}
\end{table}

\subsection{Recovered potentials are interpretable}
\label{sec:real-interp}
The graph-Sobolev potential ($\lambda_1=1$) reproduces the known funnel ordering
(\cref{fig:real-potentials}, \cref{app:figures}). On RetailRocket,
\texttt{view} ($-0.60$) sits below the abandon origin and \texttt{addtocart}
($+0.32$) between the boundaries; on Trivago, clickout is highest and searches
lowest, with one flagged anomaly (\texttt{interaction item image}, $-1.52$).
On OTTO (the only genuinely multi-sink graph) the solve separates \texttt{carts}
(mean $+0.26$, straddling the cart sink) from \texttt{clicks} (mean $-0.01$),
recovering $\text{clicks}<\text{carts}<\text{orders}$ from flow alone.

\subsection{The regularizer contrast holds without ground truth}
\label{sec:real-contrast}
With no truth, we measure how each penalty at $\lambda_1=1$ perturbs the
unregularized solve along two axes: interior \emph{ordering} (Spearman vs.\ the
base) and \emph{dynamic range} (fraction retained). \Cref{tab:real-contrast,fig:real-contrast} reproduce the synthetic finding:
graph-Sobolev retains $28$--$41\%$ of the range across all three corpora, while
ridge retains $21\%$, $5\%$ and, on the $105$-state OTTO graph, just $0.2\%$. The
contrast sharpens with graph size; on OTTO ridge retains essentially none (rank
agreement $0.42$ vs.\ $0.66$), as \cref{prop:tik-collapse} predicts.

\subsection{Residual concentration and stability}
\label{sec:real-stability}
The Poisson residual ranks the known conversion sink in the top three of up to
$105$ states (RetailRocket $3/4$, Trivago $1/11$, OTTO $2/105$), confirming
\cref{thm:residual}; being single-sink, these corpora do not exercise its
multi-sink advantage. Bootstrap resampling ($20$ resamples) gives small
coefficients of variation on well-visited states (median $0.010$, $0.021$,
$0.127$), so $\phi$ is stable wherever the divergence is well-sampled.

\subsection{The recovered potential is a usable node feature}
\label{sec:real-downstream}
Does the potential carry \emph{downstream} signal? We predict OTTO session
conversion with a logistic regression on four per-session order statistics (mean, max, min, last) of the interior states' potentials; the task is
leakage-guarded, as the label is the terminal sink, the features exclude every
sink, and the potential is fit on the training sessions only
(\cref{app:downstream}). The graph-Sobolev potential adds $+3.1$ ROC-AUC
points over raw session statistics ($0.828\!\to\!0.859$; \cref{tab:downstream})
while ridge adds nothing ($0.829$); standalone, the $4$-D Sobolev feature scores
$0.836$ against ridge's $0.717$: the gauge-invariant penalty yields a usable
(if compressed) node feature where the collapsed ridge potential does not. The
gap \emph{widens} on a ten-times larger real graph:
re-clustering OTTO into $1{,}005$ states gives $0.858$ for graph-Sobolev against
$0.496$ (chance) for ridge (\cref{tab:otto1005}, \cref{app:figures}).

\section{Conclusion}
\label{sec:conclusion}

The standard regularizer for directed Poisson inverse problems fails
specifically: ridge (the identity seminorm) imposes a meaningless origin and, as
its strength grows, collapses the dynamic range and reverses the ordering, so the
only safe setting is none. The gauge-invariant
Dirichlet energy removes the hazard, with one defining consequence,
\emph{parameter-insensitivity}: the recovered potential is stable across four
orders of magnitude in $\lambda_1$, whereas ridge inverts the ordering for every
$\lambda_1>0$. We proved gauge invariance (\cref{prop:gauge}), the rank structure
of the reduced system (\cref{thm:spd}) and exact range preservation on chains
(\cref{thm:range}), gave a Poisson-residual boundary diagnostic, and positioned
the contribution against general-form Tikhonov and graph smoothness priors. The
penalty is classical; its necessity here is not. The guarantee is preservation,
not signal: whatever ordering the data supports, the gauge-invariant penalty
keeps it across $\lambda_1$ and ridge does not.

\paragraph{Future work.}
Characterizing when the divergence carries enough harmonic field for recovery
(a theorem for \cref{sec:exp-regime});
nonlinear constitutive laws; edge-preserving gauge-invariant penalties (total
variation, the $p$-Laplacian); and which acyclic supports best aid recovery
(\cref{sec:exp-ablation}). All experiments, figures and tables reproduce from the released code with one
fixed-seed command.

\bibliographystyle{plainnat}
\bibliography{refs}

\appendix
\section{Extraction Mechanics}
\label{app:extract}

This appendix gives the edge-retention rule deferred from \cref{sec:dag}, for
both orientations. Write the oriented coordinate as $\psi$: for the topological
sort $\psi$ is the mean visit position (or net inflow); for the Helmholtz--Hodge
projection $\psi=\phi_0$, the gradient potential solving
$G^\top G\,\phi_0=G^\top\omega$. With the gradient-dominance ratio
\begin{equation}
  \delta(u,v)=\frac{\lvert\psi(v)-\psi(u)\rvert}{\lvert\omega_{uv}\rvert+\epsilon},
  \qquad \omega_{uv}=F_{uv}-F_{vu},
\end{equation}
the retained edge set is
\begin{equation}
  E_\delta=\bigl\{(u,v): \psi(v)>\psi(u)\ \text{and}\ \delta(u,v)\ge\tau\bigr\},
  \label{eq:retain}
\end{equation}
followed by top-$k$ pruning per source (retain each node's $k$ heaviest outgoing
edges). The threshold $\tau$ trades support density against strictness: a larger
$\tau$ keeps only edges whose oriented drop dominates the raw flow, and on small
dense graphs (for example the eleven-state event graph of \cref{sec:real}) a
lower $\tau$ is needed to avoid disconnecting interior states. All reported runs
use $\rho=2$, $k=10$, and $\tau$ as stated per experiment.

\paragraph{Trace normalization and solver cost.}
\label{sec:sob-trace}
To make $\lambda_1$ comparable across graphs we normalize the penalty block by its
trace, $\lambda_1(\LG)_{\interior\interior}\mapsto
\lambda_1(\LG)_{\interior\interior}/\Tr[(\LG)_{\interior\interior}]$; all sweeps
use this. The dense solve is $O(\lvert\interior\rvert^3)$: Cholesky would suffice
under \cref{thm:spd} ($M$ SPD), but real extractions violate that hypothesis
(\cref{eq:kerM}), so we use the symmetric eigensolve throughout to return the
minimum-norm solution. It is negligible for the tens-to-hundreds of interior
states here; a conjugate-gradient solve at $O(\kappa(M)\,\mathrm{nnz})$ is
available when $\lvert\interior\rvert$ is large.

\section{Proofs}
\label{app:proofs}

\paragraph{The chain instrument.}
Both proofs use the chain implemented in the released code. The nodes are
$0,1,\dots,L+1$ with unit-conductance edges $i\to i+1$ for $i=0,\dots,L$. The
sink $L+1$ is the only Dirichlet node, pinned at $\phi_{L+1}=1$; unit
divergence is injected at the source and absorbed at the sink,
\[
  b_0=-1,\qquad b_{L+1}=+1,\qquad b_i=0\ (1\le i\le L).
\]
The interior unknown is $x=(\phi_0,\dots,\phi_L)$. The directed Laplacian acts
as $(\Lap\phi)_i=\phi_i-\phi_{i+1}$ for $i=0,\dots,L$ (and $(\Lap\phi)_{L+1}=0$
since the sink has no out-edge). Writing the interior block $\Lap_{\interior
\interior}$ and eliminating the pinned sink gives the effective right-hand
side $b^{\mathrm{eff}}=b_\interior-\Lap_{\interior\bdry}\phi_\bdry$, which has a
$-1$ in coordinate $0$ and a $+1$ in coordinate $L$ and is zero elsewhere. It
is convenient to use the \emph{edge gaps}
\[
  d_i \;=\; \phi_i-\phi_{i+1},\qquad i=0,\dots,L,\qquad
  \text{so } \phi_i = 1+\textstyle\sum_{j=i}^{L} d_j ,
\]
since $\phi_{L+1}=1$. A short computation gives
$(\Lap_{\interior\interior}x-b^{\mathrm{eff}})_0=d_0+1$,
$(\,\cdot\,)_i=d_i$ for $1\le i\le L$, so
\begin{equation}
  \norm{\Lap_{\interior\interior}x-b^{\mathrm{eff}}}_2^2
  =(d_0+1)^2+\sum_{i=1}^{L} d_i^2,
  \qquad
  \phi^\top\LG\phi=\sum_{i=0}^{L} d_i^2 .
  \label{eq:chain-forms}
\end{equation}
The interior dynamic range is $\Delta\phihat=\max_{0\le i\le L}\phi_i
-\min_{0\le i\le L}\phi_i$.

\subsection{Proof of \texorpdfstring{\cref{prop:tik-collapse}}{Proposition 4}}
Ridge minimizes $\tfrac{1}{2}\norm{\Lap_{\interior\interior}x-b^{\mathrm{eff}}}_2^2
+\lambda_1\norm{x}_2^2$ (up to the positive data-normalization constant, which
only rescales $\lambda_1$), with normal equations $(A+\lambda_1 I)x=c$, where
$A=\Lap_{\interior\interior}^\top\Lap_{\interior\interior}\succeq 0$ and
$c=\Lap_{\interior\interior}^\top b^{\mathrm{eff}}$. Because $A\succeq 0$, the
matrix $(A+\lambda_1 I)^{-1}$ has spectral norm at most $1/\lambda_1$, so
\[
  \norm{x}_2=\norm{(A+\lambda_1 I)^{-1}c}_2\le \frac{\norm{c}_2}{\lambda_1},
  \qquad\text{hence}\qquad
  \Delta\phihat_{\mathrm{Tik}}\le 2\norm{x}_\infty\le 2\norm{x}_2
  \le \frac{2\norm{c}_2}{\lambda_1}.
\]
It remains to bound $\norm{c}_2$ independently of $L$. The matrix
$\Lap_{\interior\interior}$ is upper bidiagonal with $1$ on the diagonal and
$-1$ on the superdiagonal (and a lone $1$ in the last row), so column $j$ of
$\Lap_{\interior\interior}$ has entries $(\Lap_{\interior\interior})_{jj}=1$
and $(\Lap_{\interior\interior})_{j-1,j}=-1$. Thus
$c_j=(\Lap_{\interior\interior}^\top b^{\mathrm{eff}})_j
=b^{\mathrm{eff}}_j-b^{\mathrm{eff}}_{j-1}$, and with
$b^{\mathrm{eff}}=(-1,0,\dots,0,1)$ this gives
$c=(-1,\,1,\,0,\dots,0,\,1)^\top$, so $\norm{c}_2=\sqrt 3$ for every $L$.
Therefore $\Delta\phihat_{\mathrm{Tik}}(\lambda_1)\le 2\sqrt3/\lambda_1\to0$,
with a constant independent of $L$. \qed

\subsection{Proof of \texorpdfstring{\cref{thm:range}}{Theorem 2}}
With the trace normalization of \cref{sec:sob-trace}, the Sobolev objective is
\[
  f(x)=\frac{1}{2\,b_{\mathrm{sc}}}\Big[(d_0+1)^2+\sum_{i=1}^{L} d_i^2\Big]
  +\frac{\lambda_1}{2\,\tau_L}\sum_{i=0}^{L} d_i^2,
\]
using \cref{eq:chain-forms}, where $b_{\mathrm{sc}}=\norm{b^{\mathrm{eff}}}_2^2=2$
is the data-normalization constant and
$\tau_L=\Tr[(\LG)_{\interior\interior}]=1+2(L-1)+2=2L+1$ is the penalty trace
(node $0$ has degree $1$, nodes $1,\dots,L$ have degree $2$). The crucial point
is that $f$ is \emph{separable} in the gaps $d_0,\dots,d_L$: each $d_i$ appears
in exactly one data term and one penalty term. Minimizing term by term,
\[
  d_i^\star=0\ \ (1\le i\le L),
  \qquad
  d_0^\star=\arg\min_{d}\ \frac{(d+1)^2}{2b_{\mathrm{sc}}}
  +\frac{\lambda_1}{2\tau_L}d^2
  =\frac{-1}{\,1+ \lambda_1 b_{\mathrm{sc}}/\tau_L\,}
  =\frac{-(2L+1)}{\,2L+1+2\lambda_1\,},
\]
where we used $b_{\mathrm{sc}}=2$. Since $d_1^\star=\dots=d_L^\star=0$, the
recovered potential is constant on $\{1,\dots,L+1\}$ and drops by $|d_0^\star|$
at the source, so
\[
  \Delta\phihat_{\mathrm{Sob}}(\lambda_1)=|d_0^\star|
  =\frac{2L+1}{\,2L+1+2\lambda_1\,}
  =\frac{1}{1+\lambda_1\kappa(L)},\qquad \kappa(L)=\frac{2}{2L+1}.
\]
The bound $\tfrac{1}{1+\lambda_1\kappa(L)}\ge 1-\lambda_1\kappa(L)$ follows from
$\tfrac1{1+t}\ge 1-t$ for $t\ge0$, and $\Delta\phitrue=1$ is the
$\lambda_1=0$ value. As $L\to\infty$, $\kappa(L)\to0$ and
$\Delta\phihat_{\mathrm{Sob}}\to1$. \qed

\paragraph{Numerical confirmation.}
The released solver reproduces $\Delta\phihat_{\mathrm{Sob}}=(2L+1)/(2L+1+2\lambda_1)$
to four decimals at every tested $(L,\lambda_1)$ (for example $61/63=0.968$ at
$(L,\lambda_1)=(30,1)$ and $61/81=0.753$ at $(30,10)$), and confirms that
$\Delta\phihat_{\mathrm{Tik}}$ is independent of $L$ (identical to four decimals
for $L=30$ and $L=100$ at each $\lambda_1$), as \cref{prop:tik-collapse}
predicts.

\subsection{Remark on the empirical Tikhonov plateau}
\Cref{prop:tik-collapse} gives $\Delta\phihat_{\mathrm{Tik}}\to0$ but makes
no claim about the limiting rank correlation. On a graph with nonzero
interior divergence the collapsed estimate is not identically constant: the
$O(1/\lambda_1)$ leading term is $\propto c=\Lap^\top b^{\mathrm{eff}}$, whose
sign pattern reflects local divergence rather than global position between
boundaries. The rank correlation with $\phitrue$ therefore tends to a
regime-dependent negative plateau (empirically ${\approx}-0.42$ in
\cref{tab:sweep}), not to $-1$: the recovered order is anticorrelated with the
planted field but not its exact reverse. The linear correlation likewise crosses
to zero as the range collapses (\cref{tab:sweep}).

\subsection{Proofs of the structural results}
\label{app:struct-proofs}

\begin{proof}[Proof of \cref{prop:acyclic} (acyclicity)]
By \cref{eq:retain}, every retained edge $(u,v)$ has $\psi(v)>\psi(u)$ for the
oriented coordinate $\psi$. A directed cycle $v_1\to\cdots\to v_k\to v_1$ would
force $\psi(v_1)<\cdots<\psi(v_k)<\psi(v_1)$, a contradiction; hence $\psi$
strictly increases along every directed path and is a topological order.
\end{proof}

\begin{proof}[Proof of \cref{prop:nilpotent} (nilpotency)]
$(A^m)_{uv}$ sums over directed walks of length $m$ from $u$ to $v$. On a DAG
every walk is a path, of length at most $L$; hence for $m>L$ no such walk
exists and $A^m=0$.
\end{proof}

\begin{proof}[Proof of \cref{prop:dissipation} (minimum dissipation)]
Introduce a multiplier $\phi$ for the divergence constraint and form
$\mathcal L=\tfrac12\sum W_{ij}^{-1}q_{ij}^2+\phi^\top(\divg(q)-b)$.
Stationarity in $q_{ij}$ gives $W_{ij}^{-1}q_{ij}=\phi_i-\phi_j$ up to
orientation, i.e.\ $q_{ij}=W_{ij}(\phi_j-\phi_i)$, so the minimizer is a gradient
flow. The multiplier $\phi$ is determined only up to the constant that
\cref{sec:poisson-dirichlet} pins; the forward operator acting on it is fixed by
the edge orientations, not by this variational argument: $-\Lap$ is the
generator of the walk with rates $W$ along those orientations, so $\Lap\phi=b$ is
that walk's Poisson equation with source $b$. For the empirical divergence,
$b_u$ is exactly the number of sessions ending at $u$ minus those starting
there.
\end{proof}

\begin{proof}[Proof of \cref{lem:saturation} (chain saturation)]
The forward operator is one-sided (\cref{eq:directed-lap}): $(\Lap\phi)_u$ sums
only over edges \emph{leaving} $u$. Each chain node $v_i$ has the single out-edge
$v_i\to v_{i+1}$ (with $v_{k+1}=s$), so
$(\Lap\phi)_{v_i}=w_{v_iv_{i+1}}(\phi_{v_i}-\phi_{v_{i+1}})$. Setting this equal
to $b_{v_i}=0$ and using $w_{v_iv_{i+1}}>0$ gives
$\phi_{v_i}=\phi_{v_{i+1}}$; backward induction from the pinned $\phi(s)$ yields
$\phi(v_1)=\cdots=\phi(v_k)=\phi(s)$.

We stress that this is a property of the \emph{directed} operator, not of a
two-sided flux balance. Imposing instead the symmetric net divergence
$W(\phi_{v_{i-1}}-\phi_{v_i})=W(\phi_{v_i}-\phi_{v_{i+1}})$ would force constant
\emph{differences}, i.e.\ a linear ramp, not a constant potential; the
distinction is exactly \cref{eq:directed-lap}, and it is the one-sided operator
whose interior kernel is the absorption probability
(\cref{sec:poisson-ops}) that the instrument plants.

The two readings separate numerically on the unit-conductance chain of the
released code. With $b=0$ on the interior and the sink pinned at $1$, the
one-sided operator returns $\phi=(1,1,1,1,1)$, while imposing the symmetric net
divergence returns the ramp $(0.2,0.4,0.6,0.8)$. On a small branching DAG the
planted absorption probability has interior residual $(0,0,0)$ under the
one-sided operator and $(0,+0.31,-0.31)$ under the symmetric one, confirming
that it is the one-sided operator the planted field satisfies.
\end{proof}

\begin{proof}[Proof of \cref{prop:gauge} (gauge invariance)]
Since $\LG\onevec=0$,
$(\phi+c\onevec)^\top\LG(\phi+c\onevec)=\phi^\top\LG\phi+2c\,\onevec^\top\LG\phi
+c^2\onevec^\top\LG\onevec=\phi^\top\LG\phi$. By contrast
$\norm{\phi+c\onevec}_2^2=\norm{\phi}_2^2+2c\,\onevec^\top\phi+c^2\lvert V\rvert$
depends on $c$, with a unique minimizing gauge $c^\star=-\bar\phi$.
\end{proof}

\begin{proof}[Proof of \cref{thm:spd} (definiteness and rank structure)]
Write $\tilde x$ for the extension of $x\in\R^{\interior}$ by zero on $\bdry$.

\emph{Block diagonality.} By \cref{eq:directed-lap}, row $u$ of
$\Lap_{\interior\interior}$ is supported on $\{u\}$ together with the out-neighbours
of $u$, all of which lie in $u$'s undirected component; the same holds for
$(\LG)_{\interior\interior}$. Hence both, and therefore $M$, are block diagonal
over $\mathcal{C}$.

\emph{Null space.} Since
$x^\top Mx=\norm{\Lap_{\interior\interior}x}_2^2
+\lambda_1\,x^\top(\LG)_{\interior\interior}x$ is a sum of two nonnegative terms,
$Mx=0$ iff both vanish. The quadratic form of the penalty block is the full
Dirichlet energy of the zero-extension,
$x^\top(\LG)_{\interior\interior}x=\sum_{(u,v)\in E}w_{uv}(\tilde x_u-\tilde x_v)^2$,
which vanishes iff $\tilde x$ is constant on each $C\in\mathcal{C}$; on a
component meeting $\bdry$ that constant is $0$ because $\tilde x$ vanishes on
$\bdry$. Hence $x=\sum_{C\cap\bdry=\emptyset}\alpha_C\onevec_C$. Conversely each
such $x$ also satisfies $\Lap_{\interior\interior}x=0$: for $u\in C$ with
$C\cap\bdry=\emptyset$ every out-neighbour of $u$ lies in $C\subseteq\interior$,
so $(\Lap_{\interior\interior}x)_u=\sum_v w_{uv}(\alpha_C-\alpha_C)=0$; and for
$u$ outside every boundary-free component, $x_u=0$ while no out-neighbour of $u$
lies in such a component (it would place $u$ there too), so the row vanishes as
well. This gives \cref{eq:kerM}, and positive definiteness on components meeting
$\bdry$ follows immediately.

\emph{Gauge.} If some component meets $\bdry$ and contains an interior state,
$\onevec$ is $1$ there while every element of $\ker M$ vanishes there, so
$\onevec\notin\ker M$: the gauge direction lies in $\operatorname{range}(M)$.
Claims (i) and (ii) are then the definition of $M^{+}$ acting blockwise: exact inversion on the definite blocks, zero component along each $\onevec_C$ on the others. For (iii), note $\Lap\onevec=0$, so
$\Lap_{\interior\bdry}\onevec_\bdry=-\Lap_{\interior\interior}\onevec_\interior$;
shifting the Dirichlet data by $c$ therefore sends
$b^{\mathrm{eff}}\mapsto b^{\mathrm{eff}}+c\,\Lap_{\interior\interior}\onevec$,
so the residual $\Lap_{\interior\interior}(x+c\onevec)-b^{\mathrm{eff}}$ is
unchanged, while the reduced penalty
$x^\top(\LG)_{\interior\interior}x+2\,x^\top(\LG)_{\interior\bdry}\phi_\bdry$
changes only by a constant independent of $x$, since
$(\LG)_{\interior\bdry}\onevec_\bdry=-(\LG)_{\interior\interior}\onevec_\interior$
(\cref{prop:gauge}) cancels its terms linear in $x$. The shifted minimiser is thus $x+c\onevec$ on the definite
blocks, with the data normalization held fixed; letting it float rescales
$b_{\mathrm{sc}}$ on the instrument of \cref{sec:exp}, i.e.\ an immaterial change
of effective $\lambda_1$ under parameter-insensitivity (\cref{sec:sob-empirics}):
the recovered ordering is unchanged at every shift of \cref{tab:gauge}.

\emph{The SPD special case.} If every interior node has a directed path to
$\bdry$ then every node is undirectedly connected to $\bdry$, so no component is
boundary-free, $\ker M=\{0\}$ by \cref{eq:kerM}, and $M$ is SPD.
\end{proof}

\paragraph{Numerical confirmation.}
On the instrument of \cref{sec:exp} ($\lvert\interior\rvert=271$,
$\lambda_1=1$) the released solver reproduces every clause of \cref{eq:kerM}:
$\dim\ker M=128$, matching the $128$ boundary-free components exactly, and the
two subspaces coincide to machine precision (smallest principal cosine
$1.000000$). The block on components meeting $\bdry$ is definite,
$\lambda_{\min}=5.7\times10^{-6}>0$, recovering the SPD case there; and the gauge
mode lies in the range, $\norm{M\onevec}=3.6\times10^{-2}\neq0$, which is the
algebraic reason the minimum-norm solve reproduces a boundary shift exactly
(\cref{tab:gauge}).

\begin{proof}[Proof of \cref{thm:residual} (residual concentration)]
For an absorbing sink $u$ with no out-edges,
$[\Lap\phi]_u=\sum_{j:(u,j)\in E}W_{uj}(\phi_u-\phi_j)$ has no terms, so it is
$0$ and $r_u=\lvert b_u\rvert$. For an interior node
$r_u=\lvert[\Lap\phi]_u-b_u\rvert\le\eta$ by hypothesis, and
$\min_\bdry\lvert b_u\rvert>\eta$ then orders all sinks above all interior
nodes.
\end{proof}

\section{Additional Figures and Tables}
\label{app:figures}

\begin{table}[h]
\centering
\caption{Solver scaling: sparse conjugate gradients (CG) versus the dense
eigensolve for the graph-Sobolev solve ($\lambda_1=1$) on layered DAGs of growing
interior size. CG matches the dense solve to $10^{-9}$ where both run, scales
near-linearly in edges, and handles $4\times10^4$ nodes in under half a second;
the dense solve exhausts memory beyond ${\sim}4{,}000$ interior nodes. This
backs the CG path of \cref{sec:exp-scaling} (\cref{thm:spd}). Times are
single-core wall-clock and reproduced by the released \texttt{run\_scaling}
script.}
\label{tab:scaling}
\begin{tabular}{r r r r r}
\toprule
interior nodes & edges & CG (s) & eigensolve (s) & $\max\lvert\Delta\rvert$\\
\midrule
$1{,}000$   & $1{,}980$   & $0.014$ & $0.19$  & $2.7\times10^{-9}$\\
$2{,}000$   & $3{,}960$   & $0.015$ & $1.67$  & $4.0\times10^{-9}$\\
$4{,}000$   & $7{,}920$   & $0.032$ & $24.23$ & $3.0\times10^{-9}$\\
$8{,}000$   & $15{,}840$  & $0.136$ & \emph{out of memory} & n/a\\
$16{,}000$  & $31{,}840$  & $0.209$ & \emph{out of memory} & n/a\\
$40{,}000$  & $79{,}800$  & $0.593$ & \emph{out of memory} & n/a\\
$100{,}000$ & $199{,}500$ & $1.333$ & \emph{out of memory} & n/a\\
$200{,}000$ & $399{,}500$ & $3.564$ & \emph{out of memory} & n/a\\
$400{,}000$ & $799{,}200$ & $8.473$ & \emph{out of memory} & n/a\\
\bottomrule
\end{tabular}
\end{table}

\begin{table}[h]
\centering
\caption{Regularizer contrast on real data at $\lambda_1=1$, versus the
unregularized solve (exact values for \cref{fig:real-contrast}). ``Range ret.''
is interior dynamic range as a fraction of the base; ``Rank agree.'' is the
Spearman correlation of the interior ordering with the base.}
\label{tab:real-contrast}
\begin{tabular}{l c cc cc}
\toprule
& & \multicolumn{2}{c}{graph-Sobolev} & \multicolumn{2}{c}{Tikhonov}\\
\cmidrule(lr){3-4}\cmidrule(lr){5-6}
Dataset & states & Range ret. & Rank agree. & Range ret. & Rank agree.\\
\midrule
RetailRocket & $4$   & $0.407$ & $1.000$ & $0.208$ & $1.000$\\
Trivago      & $11$  & $0.276$ & $0.883$ & $0.053$ & $0.833$\\
OTTO         & $105$ & $0.279$ & $0.662$ & $0.002$ & $0.419$\\
\bottomrule
\end{tabular}
\end{table}

\begin{table}[h]
\centering
\caption{Downstream utility on OTTO: session-level conversion prediction
(\cref{sec:real-downstream}), held-out ROC-AUC (mean$\pm$sd over three splits;
$27\%$ conversion base rate, $186{,}271$ sessions). Features use only the
interior (non-sink) states a session visits; the label is the terminal sink, so
the task is leakage-free. The graph-Sobolev potential is a usable node
feature (adding $+3.1$ points over raw statistics and beating the ridge potential by $+11.8$ points standalone), whereas the collapsed ridge potential
adds nothing. A trained directed GNN (\texttt{run\_gnn}, mean/in/out message
passing, node embeddings pooled over the same visited states) is stronger
($0.861$) but \emph{oversmooths with depth}, its AUC falling to $0.807$ at
$32$ layers as node-embedding Dirichlet energy collapses; neutralizing the
constant/gauge mode each layer holds it flat ($0.854$): the gauge principle of
\cref{sec:related} applied to the network, recovering PairNorm
\citep{zhao2020pairnorm}. Bag-of-states (full $105$-dim visit counts) is the
state-identity ceiling. Architecture and gauge-centering in \cref{app:gnn};
reproduced by the released \texttt{conversion\_prediction} and \texttt{run\_gnn}
experiments.}
\label{tab:downstream}
\begin{tabular}{l c}
\toprule
Feature set & test ROC-AUC\\
\midrule
Raw session statistics (length, \#distinct)  & $0.828\pm0.001$\\
\quad$+$ graph-Sobolev potential ($4$-D)      & $\mathbf{0.859\pm0.000}$\\
\quad$+$ Tikhonov potential ($4$-D)           & $0.829\pm0.001$\\
graph-Sobolev potential alone ($4$-D)         & $0.836\pm0.000$\\
Tikhonov potential alone ($4$-D)              & $0.717\pm0.001$\\
\midrule
Directed GNN, $2$ layers (trained)            & $0.861$\\
Directed GNN, $32$ layers (oversmoothed)      & $0.807$\\
Directed GNN, $32$ layers (gauge-centered)    & $0.854$\\
Bag-of-states ceiling ($105$-D)               & $0.899\pm0.001$\\
\bottomrule
\end{tabular}
\end{table}

\paragraph{Downstream protocol (\cref{tab:downstream}).}
\label{app:downstream}
Each session is labelled by its terminal sink (converted iff it ends at
\texttt{orders}) and featurised only from the interior (non-sink) states it
visits, so the pinned label cannot leak into the features. The potential
features are the four order statistics (mean, max, min, and last-visited $\hat\phi$) over those states; the raw baseline is
$[\#\text{visits},\ \#\text{distinct states}]$, and ``raw$+$'' rows concatenate
the two. The classifier is an $L_2$-regularised logistic regression
(scikit-learn defaults, $C{=}1$); we hold out $30\%$ of sessions and average over
three random splits, and, to avoid leakage across the split, fit the potential
on the training sessions only, then apply it to both. Because $\hat\phi$ is a
function of the state, the potential features are a $1$-D-per-state
\emph{compression} of node identity, and therefore sit below the bag-of-states
ceiling ($0.899$) by construction: they cannot exceed the full $105$-dimensional
identity. The finding is not that they do, but that the gauge-invariant
compression recovers about a third of the raw-to-ceiling AUC gap in four
features while the ridge compression recovers essentially none; a learned
encoder (e.g.\ GraphSAGE) on $\hat\phi$ could close the remainder. Reproduced by
the released \texttt{conversion\_prediction} experiment.

\paragraph{Directed-GNN baseline and gauge-centering (\cref{tab:downstream}).}
\label{app:gnn}
The GNN rows use a directed graph convolutional network on the extracted DAG
$G_\delta$ (topological-sort support, $N=105$ nodes). With learnable node
embeddings $H^0\in\R^{N\times d}$ ($d=32$), each layer aggregates from in- and
out-neighbours separately,
\begin{equation}
  H^{l+1} = \mathrm{ReLU}\!\left(\hat A_{\mathrm{out}} H^l W^l_{\mathrm o}
  + \hat A_{\mathrm{in}} H^l W^l_{\mathrm i}\right),
  \label{eq:digcn}
\end{equation}
where $\hat A_{\mathrm{out}}$ is the row-normalized adjacency of $G_\delta$ and
$\hat A_{\mathrm{in}}$ that of its transpose (self-loops enter through the pure
propagation of \cref{sec:related}). A session is scored by mean-pooling the final
embeddings over the interior states it visits (the same pooling and leakage guard as the potential features above), followed by a logistic layer; we train
end-to-end with Adam (learning rate $10^{-2}$, weight decay $10^{-4}$, $80$
epochs, binary cross-entropy) on the same $70/30$ session split.

Oversmoothing drives $H^l$ toward the constant (gauge) mode, the smallest
eigenvector of the symmetric normalized Laplacian $L_{\mathrm{sym}}$, which is
$u_0\propto d^{1/2}$ (node degrees $d$). \emph{Gauge-centering} removes this
component after every layer,
\begin{equation}
  H^{l}\leftarrow H^{l}-u_0\,\bigl(u_0^\top H^{l}\bigr),\qquad \|u_0\|_2=1,
  \label{eq:gauge-center}
\end{equation}
i.e.\ it subtracts the \emph{degree-weighted} mean embedding, the per-layer analogue of $\LG\onevec=0$ (for a regular graph $u_0\propto\onevec$ and this is
the ordinary mean). The ``gauge$+$scale'' variant additionally rescales $H^l$ to
fixed Frobenius norm, recovering PairNorm \citep{zhao2020pairnorm}. Oversmoothing
is measured by the node Dirichlet energy $\mathcal E(H)=\Tr(H^\top
L_{\mathrm{sym}}H)/\|H\|_F^2$, which vanishes as representations collapse. Across
depths $\{2,4,8,16,32\}$ the vanilla GCN's test AUC falls $0.86\!\to\!0.81$ as
$\mathcal E$ collapses (to $0.52$ at $16$ layers), while gauge-centering holds
AUC flat ($0.85$ at $32$ layers). Reproduced by the released \texttt{run\_gnn}
experiment.

\begin{figure}[h]
\centering
\includegraphics[width=0.58\linewidth]{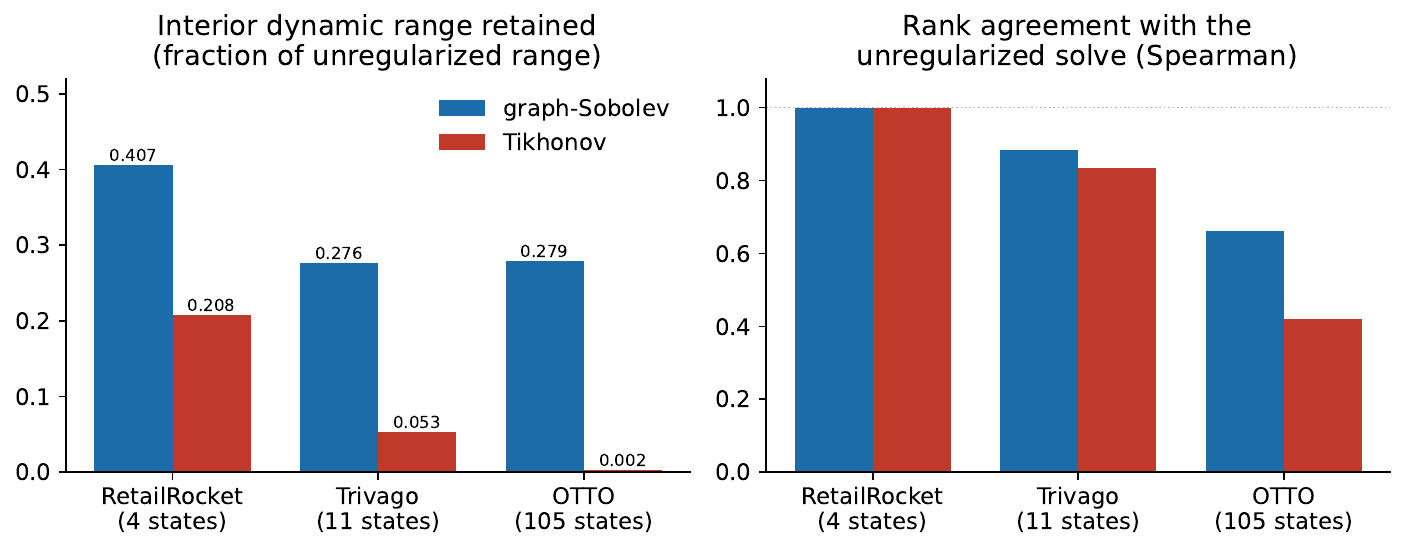}
\caption{Regularizer contrast on real data ($\lambda_1=1$) versus the
unregularized solve. Left: interior dynamic range retained. Right: rank
agreement with the base solve. Ridge's collapse deepens as the graph grows,
reaching near-total loss on OTTO; graph-Sobolev preserves both throughout.
Exact values in \cref{tab:real-contrast}.}
\label{fig:real-contrast}
\end{figure}

\begin{table}[h]
\centering
\caption{Robustness of the regularizer contrast across the instrument's
configuration space, deferred from \cref{sec:exp-regime}. Twelve variants: the
default funnel at five seeds $\{42,1,7,123,2024\}$ and seven single-axis
perturbations (branch count $\in\{3,8\}$, chain depth $\in\{30,80\}$,
abandonment $\in\{0.2,0.5\}$, sink entropy $0.3$). Entries are mean$\pm$sd of the
interior Spearman across the twelve variants at $\lambda_1=1$; ``margin'' is
graph-Sobolev$-$Tikhonov, ``min'' its smallest value over the twelve, and
``$>0$'' the count with positive margin. The contrast holds under every variant;
on the topological-sort support the graph-Sobolev ordering is itself
configuration-invariant. Reproduced by the released \texttt{regularizer\_robustness}
experiment.}
\label{tab:robustness}
\begin{tabular}{l cc cc c}
\toprule
& \multicolumn{2}{c}{Spearman vs.\ $\phitrue$} & \multicolumn{2}{c}{margin (Sob$-$Tik)} & \\
\cmidrule(lr){2-3}\cmidrule(lr){4-5}
Support & graph-Sobolev & Tikhonov & mean & min & $>0$\\
\midrule
Topological sort & $+0.99\pm0.00$ & $-0.34\pm0.06$ & $+1.34\pm0.05$ & $+1.20$ & $12/12$\\
Helmholtz--Hodge & $+0.78\pm0.17$ & $-0.38\pm0.09$ & $+1.15\pm0.22$ & $+0.45$ & $12/12$\\
\bottomrule
\end{tabular}
\end{table}

\begin{table}[h]
\centering
\caption{Gauge invariance, measured (\cref{thm:spd}(iii)). The data is fixed and
only the Dirichlet \emph{values} are shifted by a constant $c$: an
origin-invariant estimator must reproduce the shift exactly, so
$\max_\interior\lvert(\phihat-c)-\phihat_{\mathrm{base}}\rvert$ should be $0$.
Graph-Sobolev follows the boundary to ${\sim}10^{-3}$ and holds its dynamic range;
ridge does not move at all: its deviation equals the full shift, because the
$\ell_2$ penalty anchors the interior to the origin rather than to the boundary.
The residual $10^{-3}$ and the mild degradation at $c=10$ are the floating data
normalization of \cref{thm:spd}(iii), an immaterial change of effective
$\lambda_1$. Determined interior, $n=95$, $\lambda_1=1$, HHD support.}
\label{tab:gauge}
\begin{tabular}{l cc cc}
\toprule
& \multicolumn{2}{c}{graph-Sobolev} & \multicolumn{2}{c}{Tikhonov}\\
\cmidrule(lr){2-3}\cmidrule(lr){4-5}
Dirichlet values $(\text{abandon},\text{conversion})$ & deviation & range & deviation & range\\
\midrule
$(0,1)$ (reference)   & n/a                 & $1.727$ & n/a                 & $0.439$\\
$(-1,0)$, $c=-1$      & $5.0\times10^{-4}$  & $1.727$ & $1.000$             & $0.440$\\
$(0.5,1.5)$, $c=+0.5$ & $1.1\times10^{-3}$  & $1.726$ & $0.500$             & $0.439$\\
$(10,11)$, $c=+10$    & $1.2\times10^{-1}$  & $1.610$ & $10.00$             & $0.359$\\
\bottomrule
\end{tabular}
\end{table}

\begin{table}[h]
\centering
\caption{Robustness to noisy and incomplete data. \emph{Top:} Gaussian noise
added to the empirical divergence at scale $\sigma$ relative to
$\mathrm{RMS}(b)$ on the scored interior (mean$\pm$sd over five noise seeds,
extraction fixed). The \emph{unregularized} solve collapses at $5\%$ noise, while
graph-Sobolev degrades once and then holds at $+0.615$ even when the noise equals
the signal RMS: as the data term loses information the estimate falls back on the
penalty and the boundary, i.e.\ on graph geometry. \emph{Bottom:} the entire
pipeline (extraction included) re-run on session subsamples; rows are not
directly comparable, since a smaller sample yields a smaller support and a
smaller, easier scored set. Below ${\sim}2{,}000$ sessions extraction no longer
yields a usable support. HHD support, $\lambda_1=1$.}
\label{tab:noise}
\begin{tabular}{r ccc}
\toprule
$\sigma/\mathrm{RMS}(b)$ & unregularized & graph-Sobolev & Tikhonov\\
\midrule
$0.00$ & $+0.810$ & $+0.807$ & $-0.422$\\
$0.05$ & $+0.068\pm0.125$ & $+0.616\pm0.040$ & $-0.469\pm0.038$\\
$0.10$ & $+0.027\pm0.130$ & $+0.616\pm0.040$ & $-0.363\pm0.053$\\
$0.25$ & $+0.002\pm0.128$ & $+0.615\pm0.040$ & $-0.279\pm0.051$\\
$0.50$ & $-0.007\pm0.128$ & $+0.615\pm0.040$ & $-0.240\pm0.058$\\
$1.00$ & $-0.010\pm0.128$ & $+0.615\pm0.040$ & $-0.227\pm0.064$\\
\midrule
\multicolumn{4}{l}{\emph{sessions retained (full pipeline re-run)}}\\
\midrule
$5\%$   & \multicolumn{1}{c}{$2{,}000$ sess., $n=47$} & $+0.888$ & $-0.172$\\
$25\%$  & \multicolumn{1}{c}{$10{,}000$ sess., $n=66$} & $+0.855$ & $-0.217$\\
$100\%$ & \multicolumn{1}{c}{$40{,}000$ sess., $n=95$} & $+0.807$ & $-0.422$\\
\bottomrule
\end{tabular}
\end{table}

\begin{table}[h]
\centering
\caption{Boundary misspecification. The data and the scoring set ($n=95$) are
held fixed and only the boundary handed to the solver is corrupted. Mis-\emph{valued}
boundaries are harmless (pinning conversion at $0.5$ leaves the ordering exactly
unchanged, since a uniform rescaling of the boundary rescales $\phi$ and
preserves rank) and \emph{over}-specification is mild; \emph{missing} true sinks
is the damaging case, which is precisely what the Poisson-residual diagnostic of
\cref{sec:sink} detects. The regularizer contrast survives every variant.}
\label{tab:badbnd}
\begin{tabular}{l cc}
\toprule
Boundary given to the solver & graph-Sobolev & Tikhonov\\
\midrule
correct ($5$ conversion $+$ abandon)      & $+0.807$ & $-0.422$\\
missing $1$--$4$ of $5$ conversion sinks  & $+0.15$ \dots $+0.39$ & $-0.24$ \dots $-0.32$\\
abandon sink omitted entirely             & $+0.329$ & $-0.424$\\
$2$ spurious interior nodes pinned at $1$ & $+0.807$ & $-0.405$\\
$5$ spurious interior nodes pinned at $1$ & $+0.719$ & $-0.468$\\
fuzzy: conversion pinned at $0.5$         & $+0.807$ & $-0.469$\\
\bottomrule
\end{tabular}
\end{table}

\begin{table}[h]
\centering
\caption{A ten-times larger real graph. The $105$-state OTTO space of
\cref{tab:real-datasets} is a modelling choice (co-visit clusters), not a limit;
re-clustering the same corpus into $500$ categories yields a $1{,}005$-state real
directed graph ($189{,}507$ sessions, $7{,}518$ retained edges, $1{,}002$
interior states). Every contrast \emph{sharpens} with scale: the residual ranks
the conversion sink $2$nd of $1{,}005$, and as a downstream node feature the
graph-Sobolev potential is far above the ridge potential, which falls to chance.
The interior \emph{range}-retention metric becomes uninformative at this size
because the unregularized base range inflates to $438.8$, so we report rank
agreement and downstream AUC instead.}
\label{tab:otto1005}
\begin{tabular}{l cc}
\toprule
OTTO, $1{,}005$ states & graph-Sobolev & Tikhonov\\
\midrule
Rank agreement with unregularized base & $0.804$ & $0.679$\\
Downstream AUC, potential alone ($4$-D) & $\mathbf{0.858}$ & $0.496$ (chance)\\
Downstream AUC, raw $+$ potential       & $\mathbf{0.871}$ & $0.824$\\
\midrule
\multicolumn{3}{l}{raw session statistics alone: $0.824$\quad
bag-of-states ceiling ($1{,}005$-D): $0.892$}\\
\multicolumn{3}{l}{Poisson residual ranks the conversion sink $2/1{,}005$}\\
\bottomrule
\end{tabular}
\end{table}

\begin{figure}[h]
\centering
\includegraphics[width=0.7\linewidth]{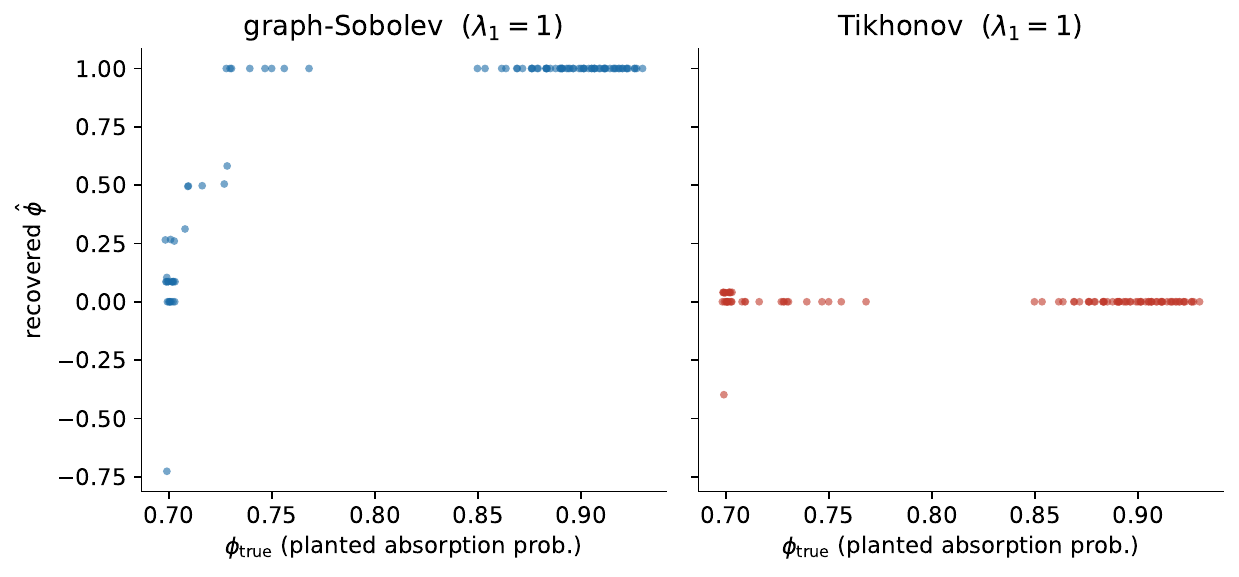}
\caption{Recovered $\phihat$ versus planted $\phitrue$ at $\lambda_1=1$
(synthetic instrument). Graph-Sobolev (left) retains a positive relationship
with the planted field; Tikhonov (right) collapses every interior state into a
flat band near the abandon origin, the range collapse of
\cref{prop:tik-collapse}.}
\label{fig:scatter}
\end{figure}

\begin{figure}[h]
\centering
\includegraphics[width=0.6\linewidth]{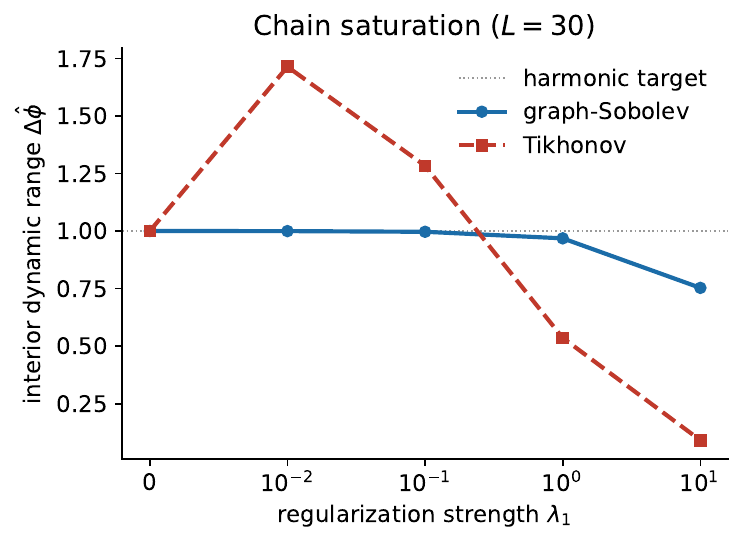}
\caption{Interior dynamic range $\Delta\phihat$ on a length-$30$ chain versus
$\lambda_1$. Graph-Sobolev stays near the unit harmonic target; Tikhonov
overshoots at small $\lambda_1$ then collapses toward zero.}
\label{fig:chain}
\end{figure}

\begin{table}[h]
\centering
\caption{Blind sink discovery on the synthetic instrument. The five planted
conversion sinks, ranked among $277$ states by Poisson residual versus by
recovered potential (lower is better); the residual places all sinks in the top
seven, invariantly across abandonment regimes.}
\label{tab:sink}
\begin{tabular}{l c c}
\toprule
Method & worst sink rank & regime ($\text{abandon}\in\{0.05,0.2,0.5\}$)\\
\midrule
Poisson residual (ours) & $7$  & invariant\\
Recovered potential     & $108$ & n/a\\
\bottomrule
\end{tabular}
\end{table}

\begin{figure}[h]
\centering
\includegraphics[width=\linewidth]{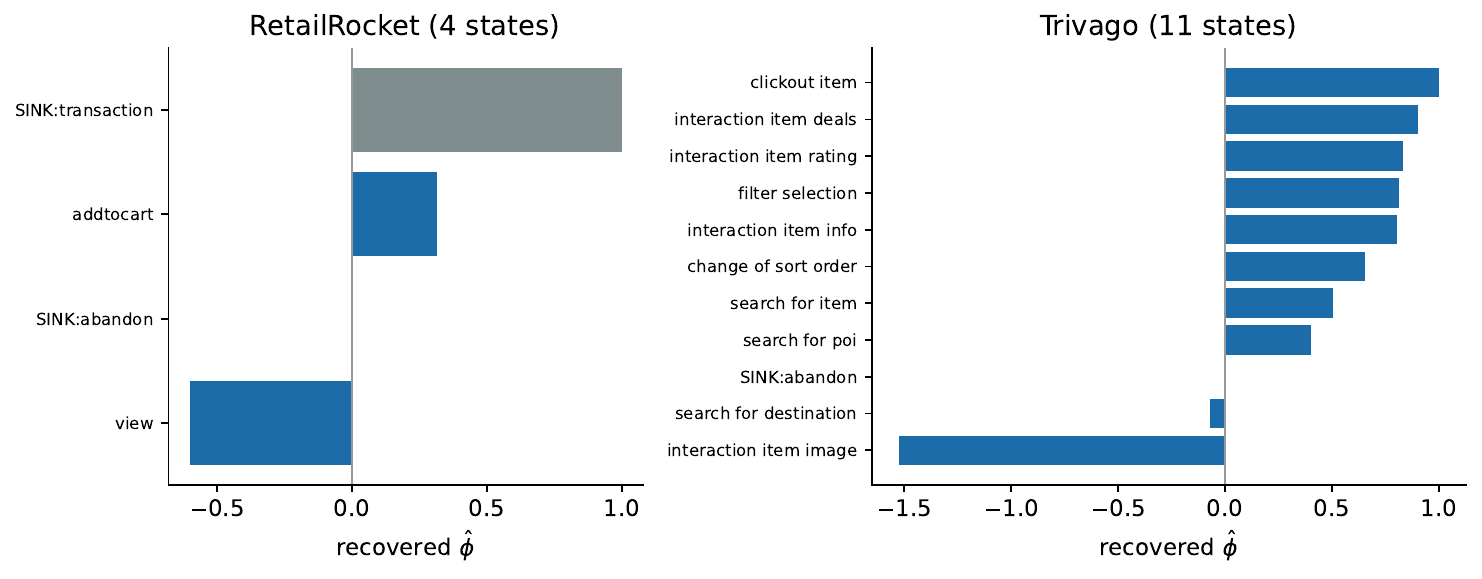}
\caption{Recovered graph-Sobolev potential ($\lambda_1=1$) on RetailRocket and
Trivago. Grey bars are pinned sinks; blue bars are recovered interior states.
The ordering reproduces the known funnel; the negative \texttt{interaction item
image} state on Trivago is discussed in \cref{sec:real-interp}.}
\label{fig:real-potentials}
\end{figure}

\begin{figure}[h]
\centering
\includegraphics[width=0.6\linewidth]{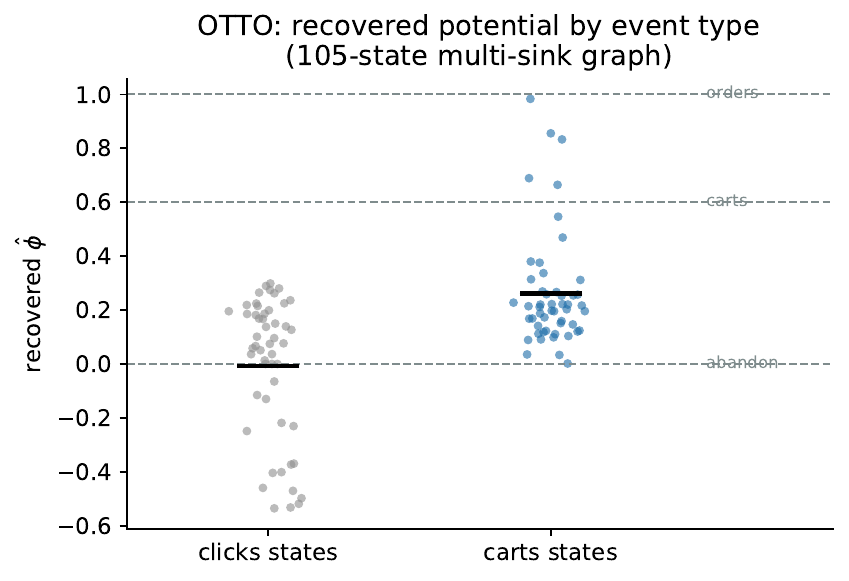}
\caption{OTTO recovered potential by event type ($105$-state multi-sink graph).
Each point is an interior state; black bars are event-type means; dashed lines
are the three Dirichlet sink levels. The solve separates \texttt{carts} states
(mean $+0.26$, straddling the cart sink) from \texttt{clicks} states (mean
$-0.01$, near abandon).}
\label{fig:otto}
\end{figure}

\section{Uncertainty Quantification}
\label{app:uq}
The divergence $b$ is an empirical estimate, and its sampling covariance
propagates to $\phi$ in closed form. On the reduced interior system
$\phi_\interior=\Lap_{\interior\interior}^{-1} b_\interior^{\mathrm{eff}}$
(at $\lambda_1=0$ for clarity), the delta method gives
\begin{equation}
  \operatorname{Cov}(\phi_\interior) \;=\;
  \Lap_{\interior\interior}^{-1}\,
  \widehat{\operatorname{Cov}}(b_\interior)\,
  \Lap_{\interior\interior}^{-\top},
\end{equation}
with $\widehat{\operatorname{Cov}}(b)$ estimated from the multinomial
transition counts. For $\lambda_1>0$ replace $\Lap_{\interior\interior}^{-1}$
by $M^{-1}\Lap_{\interior\interior}^\top$ with $M$ as in \cref{eq:M}.

\section{Reproduction}
\label{app:repro}
The released repository (\url{https://github.com/MohammadForouhesh/gauge-flow-recovery},
with a self-contained \texttt{README}) contains the core solver
(\texttt{poisson\_inverse/core.py}), the synthetic instrument
(\texttt{poisson\_inverse/synthetic.py}), the experiments
(\texttt{poisson\_inverse/experiments.py}), and two scripts:
\texttt{scripts/run\_experiments.py} writes all numerical results to JSON,
and \texttt{scripts/make\_figures.py} regenerates
\cref{fig:sweep,fig:scatter,fig:chain}. The default seed reproduces every
value reported in \cref{tab:sweep,tab:sink} and in
\cref{sec:exp-structure,sec:sob-empirics}.

The real-data results of \cref{sec:real} are reproduced by
\texttt{poisson\_inverse/loaders.py} (the RetailRocket, Trivago, and OTTO
loaders, which reduce each corpus to its event-type state space) together
with \texttt{scripts/run\_real\_data.py} (which runs the interpretability,
regularizer-contrast, residual, and bootstrap checks) and
\texttt{scripts/make\_real\_figures.py} (which regenerates
\cref{fig:real-contrast,fig:real-potentials,fig:otto} from the per-dataset
result files). Only the single raw event file per corpus is required; the
loaders construct the state space and Dirichlet boundary automatically.

\end{document}